\documentclass[twoside]{article}

\usepackage{amsmath, amssymb, amsfonts, amsthm}
\usepackage{algorithm, algpseudocode}
\usepackage[dvipsnames]{xcolor}
\usepackage{mathtools}
\usepackage{enumitem}
\usepackage{multicol}
\usepackage{caption, subcaption}
\usepackage{xr-hyper, etoolbox}
\usepackage[hidelinks]{hyperref}
\usepackage{booktabs}
\hypersetup{
  colorlinks = true, 
  citecolor  = Green, 
  linkcolor  = Red, 
  urlcolor   = NavyBlue, 
}
\newcommand{\figonecol}{0.72}
\newcommand{\figtwocol}{0.93}
\input{mysymbol.sty}
\usepackage[accepted]{aistats2024}
\usepackage[sort,round]{natbib}

\begin{document}

%

%

\twocolumn[

\aistatstitle{Estimation of partially known Gaussian graphical models with score-based structural priors}

\aistatsauthor{ Martín Sevilla \And Antonio G. Marques \And  Santiago Segarra }

\aistatsaddress{ Rice University, USA \And  King Juan Carlos University, Spain \And Rice University, USA } ]

\begin{abstract}
We propose a novel algorithm for the support estimation of partially known Gaussian graphical models that incorporates prior information about the underlying graph. 
In contrast to classical approaches that provide a \emph{point estimate} based on a maximum likelihood or a maximum a posteriori criterion using (simple) priors \emph{on the precision matrix}, we consider a prior \emph{on the graph} and rely on annealed Langevin diffusion to generate \emph{samples from the posterior distribution}.
Since the Langevin sampler requires access to the score function of the underlying graph prior, we use graph neural networks to effectively estimate the score from a graph dataset (either available beforehand or generated from a known distribution). 
Numerical experiments demonstrate the benefits of our approach.
\end{abstract}

\section{INTRODUCTION} \label{sec:introduction}
Graphical models are useful probabilistic tools represented by graphs $\ccalG = (\ccalV, \ccalE)$, where nodes $\ccalV$ encode random variables, and edges $\ccalE$ encode information regarding the variables' joint distribution.
Markov random fields (MRFs) -- an important class of graphical models -- offer the attractive property that the absence of an edge between two nodes means that the two associated random variables are conditionally independent given the rest~\citep{bishop}.
When the distribution is a multivariate Gaussian, the MRF is said to be a Gaussian Markov random field (GMRF) or Gaussian graphical model (GGM)~\citep{rue2005gaussian}. 
GGMs have been used to model complex relationships in a wide variety of disciplines, with relevant examples including gene interactions~\citep{dobra2004sparse,wang_ggm_2020}, spectrometric data~\citep{CODAZZI2022107416}, metabolic association networks~\citep{tan2017bayesian}, macroeconomic growth~\citep{dobra2010modeling}, and social networks~\citep{li2020high}.

The success of GGMs in such a broad range of datasets stems partly from their intuitive interpretability. To be specific, let $\bbA\!\in\! \{0,\!1\}^{n\times n}$ be the adjacency matrix of $\ccalG$, with $n = |\ccalV|$, and let $\bbx \! \in \! \reals^n$ be a random vector such that $\bbx\sim \ccalN(\pmb{0}, \bbSigma)$.
We then say that $\bbx$ is a GGM with respect to $\ccalG$ if and only if the precision matrix $\bbTheta=\bbSigma^{-1}$ and $\bbA$ have the same support (i.e., the same zero pattern)~\citep{rue2005gaussian}. 
As a result, the graph associated with a GGM can be estimated from the non-zero values of $\bbTheta$. 

\vspace{1mm}
\noindent {\bf Related work and limitations.}
The problem of estimating the precision matrix $\bbTheta$ (and, consequently, its support) is classically known as \textit{covariance selection}~\citep{dempster}. 
Specifically, given a set of $k$ observations $\bbx_1,\ldots,\bbx_k$, each of dimension $n$, the goal is to recover the $n\times n$ matrix $\bbTheta$ and, in particular, $\supp{\bbTheta}$, which reveals the conditional independence relationships in the GGM. Given that $p(\bbx \mid \bbTheta)$ has a closed-form expression for a GGM, a straightforward approach is to find the maximum likelihood (ML) estimator for $\bbTheta$, which is given by the inverse of the \emph{sample} covariance $\bbS$~\citep{casella2021statistical}.
However, $\bbS^{-1}$ typically does not contain exact zero entries.
Hence, one common technique is thresholding $\bbS^{-1}$~\citep{thresholding}, but this may yield unsatisfactory results if the threshold is not adequately chosen or if the number of observations $k$ is limited.

More sophisticated techniques generally propose \emph{penalties} in optimization problems that use the log-likelihood as cost function~\citep{joint_ggm_survey, williams2020beyond}.
From a Bayesian standpoint, these penalties can be considered prior distributions $p(\bbTheta)$.
For any penalty function $P(\bbTheta)$, there is an associated prior distribution $p(\bbTheta)\propto \exp(-P(\bbTheta))$ such that the penalized ML estimation boils down to
\begin{align} \label{eq:penalized_loglikelihood}
    \bbTheta_{\mathrm{est}}
    &= \argmax_{\bbTheta \succeq 0} p(\bbx_1,\ldots,\bbx_k \mid \bbTheta) p(\bbTheta) \nonumber
    \\
    & = 
    \argmax_{\bbTheta \succeq 0} \log \det\bbTheta - \tr\left(\bbS\bbTheta \right) - P(\bbTheta),
\end{align}
with the most popular penalty being $P(\bbTheta) =\lambda\sum_{i \neq j}|\Theta_{ij}|$, which translates to a Laplace prior from a Bayesian perspective.
This penalty encourages sparsity in $\bbTheta$ and is convex, rendering~\eqref{eq:penalized_loglikelihood} easy to optimize using the graphical lasso (GL) algorithm~\citep{glasso_banerjee,glasso}.
A similar approach that allows to penalize each element in $\bbTheta$ by a different value is the \emph{weighted} graphical lasso (WGL), which imposes a penalty of the form $P(\bbTheta) =\norm{\bbLambda \circ \bbTheta}_1=\sum_{ij}\Lambda_{ij}|\Theta_{ij}|$~\citep{weighted_glasso_1, weighted_glasso_2, weighted_glasso_3}.
Non-convex regularizers were also proposed in the literature~\citep{williams2020beyond}. 
While more involved, they all boil down to encouraging different forms of sparsity.

Albeit less numerous, works incorporating prior distributions that do not involve sparsity also exist.
In~\citet{ggm_metabolomic}, a base graph structure is used as a prior, which could become too restrictive as it requires information from the specific graph whose support we want to estimate.
In~\citet{gwishart_prior}, a G-Wishart prior distribution is used together with a Markov chain Monte Carlo (MCMC) sampler to determine the underlying GGM.
A similar approach is taken in~\citet{friedman2003being}, but using a standard Wishart instead.
Another framework is given in~\citet{grab}, focusing on modularity.
Even though these approaches do not assume mere sparsity, they propose pre-specified prior structures that may not apply to the graph under study. 

All techniques proposed so far for GGM estimation impose limitations in the prior knowledge that can be incorporated.
Importantly, the imposed priors are \emph{too simple} (e.g., sparsity) or \emph{ restrictive} (e.g., G-Wishart), and these are \emph{imposed on $\bbTheta$} while the natural approach would be to impose the structural priors on the underlying adjacency matrix $\bbA$.

\vspace{1mm}
\noindent {\bf Addressing these limitations.}
This paper proposes a new approach that allows the introduction of \emph{arbitrary} prior information \emph{directly} on $\bbA$, based on a dataset of adjacency matrices $\ccalA$ (of potentially varying sizes) whose distribution we use as a prior $p(\bbA)$.
This is particularly useful for real-world applications where we often have datasets of graphs instead of a closed-form prior distribution $p(\bbA)$.
For instance, when learning brain networks, leveraging graphs from other patients is feasible (and valuable) because they often share similar structures. 
This idea also applies to other areas like molecular datasets or social networks.
Additionally, many random graphs do not have a closed-form distribution, but generating samples to create a synthetic $\ccalA$ is relatively easy. 

Furthermore, we allow edge-to-edge constraints (in a WGL fashion) so that edges that are \emph{known} to either exist or be absent lead to values of $1$ or $0$ in $\bbA$, respectively.
This is useful in cases where some pairs of variables are known to be conditionally independent, but we want to estimate the rest of the graph.
Examples include gene expression data~\citep{weighted_glasso_1}, neurotoxicology tests~\citep{grzebyk2004identification}, social networks~\citep{Wu852798}, or brain networks~\citep{SIMPSON2015310}.

Our algorithm is based on Langevin dynamics, an MCMC sampler~\citep{MCMCbook, Roberts1996ExponentialCO}.
We directly sample from the posterior by defining a stochastic dynamic process whose stationary distribution matches the desired posterior distribution. 
If the interest is in a point estimate, we can readily use the samples to estimate, e.g.,  the posterior mean of the missing values in $\bbA$.

\vspace{1mm}
\noindent {\bf Contributions.}
Our three main contributions are:
\\
1) We propose a novel GGM estimator based on \emph{sampling} from a posterior distribution rather than finding the ML or MAP estimators and show that our estimator is consistent.
\\
2) We leverage annealed Langevin dynamics to implement such an estimator, which allows us to incorporate an arbitrary prior distribution learned from data. We allow the known graphs to be of different sizes, as is the case in many practical applications.
\\
3) Through numerical experiments, we show that incorporating arbitrary prior distributions outperforms estimators that only consider sparsity as prior information or are just based on the likelihood of the observations.

\vspace{1mm}
\noindent {\bf Notation.} 
Scalars, vectors, and matrices are denoted by lowercase ($y$), lowercase bold ($\bby$), and uppercase bold ($\bbY$) letters, respectively.
For a matrix $\bbY$, $Y_{ij}$ denotes its $(i,j)$-th entry.
For a vector $\bby$, its $i$-th component is represented by $y_i$.
$\bbI$ is the identity matrix of appropriate dimensions.
The operation $\circ$ denotes the Hadamard product. 
We define the element-wise indicator function as $\ind{\cdot}$, and the support of a matrix as $\supp{\bbY} = \ind{\bbY \neq 0}$ (i.e., a binary-valued matrix that is $0$ in the $(i,j)$ entries such that $Y_{ij}=0$ and is $1$ otherwise).

\section{PROBLEM FORMULATION}
We consider an unweighted and undirected graph $\ccalG$ with no self-loops that consists of $n$ nodes and a partially known set of edges.
The edge information is encoded in the adjacency matrix $\bbA \in \{0,1 \}^{n\times n}$.
To distinguish between the entries of $\bbA$ that are known and the ones we aim to estimate, we define two sets of indices $\ccalO$ and $\ccalU$ such that
\begin{eqnarray}
    \ccalO &=& \left\{(i,j) : A_{ij} \ \text{is observed} \land i < j\right\}, ~\text{and}\\
    \ccalU &=& \left\{(i,j) : A_{ij} \ \text{is unknown} \land i < j\right\} .
\end{eqnarray}
Throughout this work, we refer to the known and unknown fractions of the adjacency matrix as $\bbA^\ccalO$ and $\bbA^\ccalU$, respectively.
The condition that $i<j$ implies that we do not take into account the diagonal (since the graph has no self-loops), and we just consider the upper-triangular part of $\bbA$ (since it is symmetric).
In case the set of edges is completely unknown, we can estimate some of its entries with high confidence and consider them known, as we show in Section~\ref{sec:num_results}.

Apart from the known fraction of the graph, we also assume that $k$ independent observations are available.
We arrange them as columns in the matrix $\bbX =\begin{bmatrix} \bbx_1 \ \ldots \ \bbx_k \end{bmatrix} \in \reals^{n \times k}$.
Each observation $\bbx$ follows a normal distribution $\ccalN \left(\pmb{0}, \bbTheta_0^{-1}\right)$, where $\bbTheta_0$ is the true precision matrix.
Since $\bbA_0 = \supp{\bbTheta_0}$, then $\bbTheta_0$ is known to be $0$ where $\bbA_0^\ccalO$ is $0$ and is known to be different from $0$ where $\bbA_0^\ccalO$ is $1$.
We are also given a set of adjacency matrices $\ccalA$ drawn from the distribution $p(\bbA)$, the same distribution from which $\bbA_0$ was drawn.
In this setting, our problem is defined as follows:

\vspace{2mm}
\begin{problem}\label{prob:main}
    Given the $k$ observations $\bbX$, a partially known adjacency matrix $\bbA_0^\ccalO$, and structural prior information given by a set of matrices $\ccalA$, find an estimate of $\bbA_0^\ccalU$.
\end{problem}
\vspace{2mm}

A natural way to solve Problem~\ref{prob:main} would be to compute the MAP, forcing the entries of the estimate to be equal to those of $\bbA_0^\ccalO$ for all positions $(i,j) \in \ccalO$. Mathematically, this is given by
\begin{alignat}{2}
    \label{eq:map}
    \hbA_\MAP
    = 
    &\argmax_{\bbA} & \ & p\left(\bbX\mid\bbA \right) p\left(\bbA\right) \\
    &\text{subject to} & & A_{ij} = A^\ccalO_{0_{ij}} \,\,\,\,\,\,\, \forall (i,j) \in \ccalO. \nonumber
\end{alignat}
There are two main issues when solving~\eqref{eq:map}, which we will describe in detail next.
First, the likelihood $p\left(\bbX \mid \bbA \right)$ is not easy to calculate, since only the expression for $p(\bbX \mid \bbTheta)$ is available, which is
\begin{equation} \label{eq:likelihood_theta}
    p(\bbX \mid \bbTheta) 
    = 
    \sqrt{\frac{\det \bbTheta^k}{(2\pi)^{nk}}}
    \exp\left(-\frac{k}{2} \tr\left( \bbS \bbTheta \right)\right),
\end{equation}
where $\bbS = \frac{1}{k} \bbX \bbX^\top$ is the sample covariance.
Hence, computing $p\left(\bbX \mid \bbA \right)$ requires integrating~\eqref{eq:likelihood_theta} over all possible precision matrices such that $\supp{\bbTheta} = \bbA$, which is infeasible to do.
Second, even if $p(\bbX \mid \bbA)$ were available, carrying out the maximization in~\eqref{eq:map} would be intractable since the feasible set contains $2^{|\ccalU|}$ possible matrices.

Within the realm of point estimators, this work proposes an alternative approach to Problem~\ref{prob:main}, under which we estimate $\bbA_0$ as the posterior mean instead of the posterior mode.
That is, we aim to compute
\begin{equation}
    \E{\bbA \mid \bbX}
    = \sum_{\bbA \, \mathrm{s. t.} \, A_{ij} = A^\ccalO_{ij}} \bbA \cdot p(\bbA \mid \bbX) .
    \label{eq:posterior_mean}
\end{equation}
Note that the estimation of $\bbA_0$ can be considered a classification problem, where each edge is classified as $0$ or $1$.
Hence, choosing a thresholded version of~\eqref{eq:posterior_mean} as an estimator offers the desirable property of minimizing the edge classification error rate.
However, even if we knew $p(\bbA \mid \bbX)$, the summation in~\eqref{eq:posterior_mean} requires computing $2^{|\ccalU|}$ terms. Our approach to bypass this is to approximate~\eqref{eq:posterior_mean} by taking the sample mean across $M$ samples:
\begin{equation}
    \E{\bbA \mid \bbX} \simeq \frac{1}{M} \sum_{m=1}^M \bbA^{(m)}.
    \label{eq:sample_mean}
\end{equation}
The samples $\bbA^{(m)}$ should be drawn from the posterior
\begin{equation}
    p(\bbA \mid \bbX) 
    \propto p(\bbX \mid \bbA)p(\bbA) ,
    \label{eq:real_posterior}
\end{equation}
where we omitted conditioning on $\bbA^\ccalO$ to avoid cumbersome notation. 
As already explained, computing $p(\bbX \mid \bbA)$ in~\eqref{eq:real_posterior} is, in general, infeasible. 
As a result, rather than trying to obtain $p(\bbA \mid \bbX)$, our approach is to design an algorithm capable of sampling from \eqref{eq:real_posterior} directly without explicitly computing the posterior. 
The design of such an algorithm, which has value per se and can be used to design other point estimators, is tackled in Section~\ref{sec:main_method}.

\section{LANGEVIN FOR SUPPORT ESTIMATION}
\label{sec:main_method}
This section explains how to use annealed Langevin dynamics to solve Problem~\ref{prob:main}.
In Section~\ref{subsec:estimator}, we propose a distribution: i)  that approximates the actual posterior~\eqref{eq:real_posterior} and ii) from which samples can be drawn by leveraging Langevin dynamics.
Based on this approximate posterior, we define an estimator of $\bbA_0$ and show that it is consistent.
Sections~\ref{subsec:langevin} and~\ref{subsec:annealed_langevin} explain how annealed Langevin dynamics works and why it provides a way to incorporate prior information in the estimation $\bbA_0$ via the so-called score function. 
Then, in Section~\ref{subsec:gnn}, we study how to use the dataset $\ccalA$ as prior knowledge by training a graph neural network (GNN) whose output is directly plugged into the Langevin dynamics.
Section~\ref{subsec:algorithm} describes our final algorithm, which combines and summarizes the results of this section.
An illustration of the overall procedure is shown in Figure~\ref{fig:algorithm_summary}.
\begin{figure*}[t]
    \centering
    \includegraphics[width=\figtwocol\textwidth]{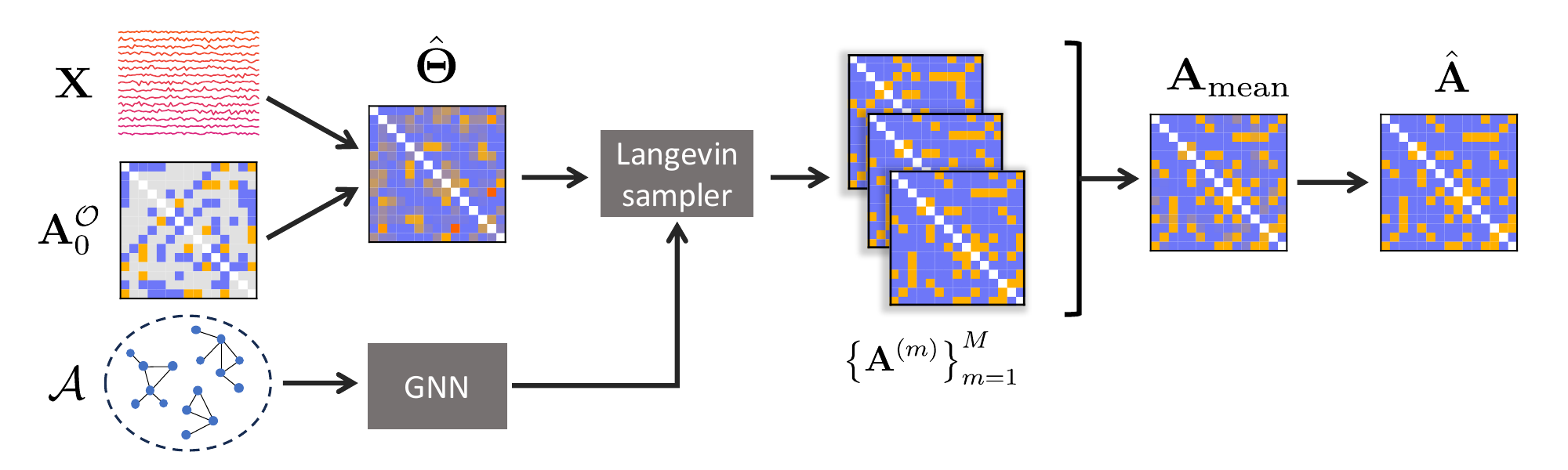}
    \caption{
    Illustration of our final algorithm [cf. Algorithm~\ref{alg:annealed_langevin}].
    The grey entries in $\bbA^{\ccalO}_0$ are what we aim to estimate (i.e., $\bbA^{\ccalU}$).
    If no entries of $\bbA_0$ are known, some can be estimated by bootstrapping $\bbX$, as shown in Section~\ref{sec:num_results}.
    By combining the GGM observations $\bbX$ and the partially observed graph $\bbA^{\ccalO}_0$, we compute the constrained ML estimator $\hbTheta$ by solving~\eqref{eq:glasso}.
    This encodes information about the likelihood of $\bbA$ given $\bbX$.
    To encode information about the prior $p(\bbA)$, we process the dataset $\ccalA$ with a GNN (Section~\ref{subsec:gnn}).
    Then, we can draw $M$ samples from the (approximate) posterior using a Langevin sampler and build any estimator with them, such as $\hbA$ in~\eqref{eq:a_est} that approximates the posterior mean.
    }
    \label{fig:algorithm_summary}
\end{figure*}

\subsection{Proposed estimator}
\label{subsec:estimator}
The first step of our algorithm consists of computing the following estimator for $\bbTheta_0$,
\begin{align} \label{eq:glasso}
    \hbTheta = &
    \argmax_{\bbTheta \succeq 0} \,\,\, \log \det\bbTheta - \tr\left(\bbS\bbTheta \right) \nonumber \\
    & \text{s. to} \,\, \Theta_{ij} = 0 \,\,\,\,\,\, \forall (i,j): A^\ccalO_{0_{ij}}=0,
\end{align}
which corresponds to the positive definite matrix that maximizes the likelihood while respecting the zero pattern known to exist.
The optimization problem in~\eqref{eq:glasso} can be efficiently solved by using the WGL algorithm mentioned in Section~\ref{sec:introduction}.
The constraint is equivalent to setting a penalty $\Lambda_{ij}$ to an arbitrarily large constant for those entries where $\bbA$ is \emph{known} to be $0$, and setting $\Lambda_{ij} = 0$ otherwise.

Let $\ccalL_\bbX\left(\bbTheta\right) = p\left(\bbX \mid \bbTheta \right)$ denote the likelihood of the precision matrix given the observed data.
Then, based on the estimator in~\eqref{eq:glasso}, 
we approximate the posterior $p(\bbA \mid \bbX)$ in~\eqref{eq:real_posterior} as
\begin{equation}
    \hhatp(\bbA \mid \bbX) \propto \ccalL_\bbX\left(\hbTheta \circ (\bbA + \bbI) \right)p(\bbA),
    \label{eq:approx_posterior}
\end{equation}
where we recall that $\circ$ is the entry-wise product. Since the entries of $(\bbA + \bbI)$ are binary, the entry-wise multiplication can be understood as a mask that sets to zero the entries of the precision that are not associated with an edge. Let us suppose now that we can sample from \eqref{eq:approx_posterior} and let $\left\{\bbA^{(m)}\right\}_{m=1}^M$ denote the set of $M$ generated independent samples. 
Then, the set $\left\{\bbA^{(m)}\right\}_{m=1}^M$ can be used to characterize the posterior. 
We focus on the posterior sample mean estimator presented in \eqref{eq:posterior_mean}--\eqref{eq:sample_mean}. 
Then, the estimator for $\bbA_0$ that we propose boils down to 
\begin{equation}
    \hbA = \ind{\left(\frac{1}{M} \sum_{m=1}^M \bbA^{(m)}\right) \geq \tau_k},
    \label{eq:a_est}
\end{equation}
where $\tau_k$ is a tunable threshold that should increase with the sample size $k$.

We aim to prove that~\eqref{eq:a_est} is a consistent estimator of $\bbA_0$, a fundamental result in our study. 
Before delving into such a proof, we establish two important intermediate results.

\begin{lemma} \label{lemma:theta}
    $\hbTheta$ as defined in~\eqref{eq:glasso} is a consistent estimator (as $k \to \infty$) of the true precision matrix $\bbTheta_0$.
\end{lemma}

\begin{proof}
    See Section~\ref{sec:proof_theta} in the Supplementary Material (SM).
\end{proof}

\begin{lemma} \label{lemma:p_hat}
    The approximate posterior $\hhatp(\bbA \mid \bbX)$ in~\eqref{eq:approx_posterior} converges in distribution to
    \begin{equation} \label{eq:phat_limit}
        \hhatp(\bbA \mid \bbX) 
        \xrightarrow{k\to\infty}
        \frac{p(\bbA)}{C} \, \delta \left(\bbTheta_0 \circ (\bbA + \bbI) - \bbTheta_0\right),
    \end{equation}
    where $C$ is a constant and $\delta(\cdot)$ is the Dirac delta.
\end{lemma}

\begin{proof}
    See Section~\ref{sec:proof_phat} in the SM.
\end{proof}

We now leverage Lemmas~\ref{lemma:theta} and~\ref{lemma:p_hat} to show consistency of $\hat{\bbA}$.

\begin{theorem}\label{thm:consistency}
    $\hbA$ as defined in~\eqref{eq:a_est} is a consistent estimator of the true adjacency matrix $\bbA_0$ when $M \to \infty$ and $\tau_k \xrightarrow{k\to\infty} 1$.
\end{theorem}

\begin{proof}
According to Lemma~\ref{lemma:p_hat}, the only matrices $\bbA$ with a positive probability of being sampled as $k \to \infty$ are those that satisfy
\begin{equation} \label{eq:false_negatives}
    \bbTheta_0 = \bbTheta_0 \circ (\bbA + \bbI).
\end{equation}

Let $\bbA^{(m)}$ be the $m$-th sample drawn from~\eqref{eq:phat_limit}.
The condition in~\eqref{eq:false_negatives} leads to
\begin{equation} \label{eq:prob_false_negative}
    \P{A^{(m)}_{ij} = 0 \mid \Theta_{0_{ij}} \neq 0} = 0 \ \forall m=1,\ldots,M.
\end{equation}
Since the estimator $\hbA$ from~\eqref{eq:a_est} is the mean of samples that follow~\eqref{eq:prob_false_negative}, for $\tau_k > 0$ we have that
\begin{equation}
    \P{\hhatA_{ij} = 0 \mid \Theta_{0_{ij}} \neq 0} = 0.
\end{equation}
On the other hand, false positives have a non-zero probability of being sampled:
\begin{equation} \label{eq:prob_false_positive}
    \P{\hhatA_{ij} = 1 \mid \Theta_{0_{ij}} = 0} 
    = 
    \P{\sum_{m=1}^M \frac{A^{(m)}_{ij}}{M} \geq \tau_k \middle| \Theta_{0_{ij}} = 0}\!\!.
\end{equation}
In the context of this proof, $\tau_k \to 1$.
Additionally, the summation in~\eqref{eq:prob_false_positive} can be at most $1$, since $A^{(m)}_{ij} \in \{0, 1\}$.
Thus,
\begin{equation} \label{eq:lim_prob_false_positive}
    \P{\hhatA_{ij} \!=\! 1 \!\mid\! \Theta_{0_{ij}} \!=\! 0} \!
    \xrightarrow{k\to\infty} \!
    \P{\sum_{m=1}^M \frac{A^{(m)}_{ij}}{M} = 1 \middle| \Theta_{0_{ij}} \!= \!0}\!\!.
\end{equation}
Another way of writing~\eqref{eq:lim_prob_false_positive} is
\begin{equation} \label{eq:lim_prob_false_positive_bis}
    \P{\sum_{m=1}^M \! \frac{A^{(m)}_{ij}}{M} \!= \! 1 \middle| \Theta_{0_{ij}} \!\!= 0}
    \!\! = \!\!
    \left(\!\P{A^{(1)}_{ij} \!= \!1 \mid  \Theta_{0_{ij}} \!= \!0} \!\right)^M\!\!\!,
\end{equation}
as each sample is drawn independently from the rest.
Since $\bbA_0 \sim p(\bbA)$, then from~\eqref{eq:phat_limit} it follows that
\begin{equation} \label{eq:prob_true_negative}
    \P{\hhatA^{(1)}_{ij} = 0 \mid \Theta_{0_{ij}} = 0} > 0.
\end{equation}
Namely, given that the true adjacency matrix $\bbA_0$ has a prior distribution $p(\bbA)$, it would not be possible for this matrix to have zero probability of being sampled from~\eqref{eq:phat_limit}.
Combining~\eqref{eq:prob_true_negative} with~\eqref{eq:lim_prob_false_positive_bis}, and then taking the limit of~\eqref{eq:lim_prob_false_positive} when $M \to \infty$ we get
\begin{equation} \label{eq:lim_lim_prob_false_positive}
    \P{\hhatA_{ij} = 1 \mid \Theta_{0_{ij}} = 0} 
    \xrightarrow{\substack{k \to \infty \\ M \to \infty}}
    0.
\end{equation}
From~\eqref{eq:lim_lim_prob_false_positive} and~\eqref{eq:prob_false_negative} it follows that, if both $M \to \infty$ and $\tau_k \to 1$, then~\eqref{eq:a_est} converges in probability to the true adjacency matrix when $k \to \infty$.
\end{proof}

In practical scenarios, infinite samples are never available, yet consistency is a desirable property for an estimator.
Furthermore, even though Theorem~\ref{thm:consistency} requires $M \to \infty$, our experiments (Section~\ref{sec:num_results}) reveal that our method outperforms classical methods for relatively small values of $M$.

To compute $\hbA$ as in~\eqref{eq:a_est} we need to be able to sample from the posterior distribution in~\eqref{eq:approx_posterior}.
To this end, we utilize the stochastic diffusion process of Langevin dynamics.

\subsection{Langevin dynamics} \label{subsec:langevin}
The Langevin dynamics algorithm is an MCMC method that allows us to draw samples from a distribution difficult to sample from directly~\citep{MCMCbook, Roberts1996ExponentialCO}.
This sampler's great advantage is that it does not require an expression for the target distribution but rather for the gradient of its logarithm.
Generically, to sample from $p(\bbw)$ via Langevin, only $\nabla_{\bbw} \log p(\bbw)$ is needed.
This gradient receives the name of \emph{score function} and is of paramount relevance in the ensuing sections.

For a generic target distribution $p(\bbw)$, the Langevin dynamics are given by
\begin{equation} \label{eq:langevin_dt}
	\bbw_{t+1} = \bbw_t + \epsilon \nabla_{\bbw}\log p(\bbw_t) + \sqrt{2\epsilon}\, \bbz_t,
\end{equation}
where $t$ is an iteration index, $\epsilon$ is the step size and $\bbz_t \sim \ccalN(0, \bbI)$. 
In each iteration, $\bbw_t$ tends to move in the direction of the score function but is also affected by white noise that prevents it from collapsing in local maxima.
Under some regularity conditions, $\bbw_t$ converges to be a sample from $p(\bbw)$ when $\epsilon \rightarrow 0$ and $t \rightarrow \infty$~\citep{wellinglang}.

It should be noted that, in our case, we are trying to sample a \emph{discrete} random vector (i.e., a vectorized unweighted adjacency matrix).
Hence, the gradient of the target log-density is not defined in our setting.
A noisy (continuous) version of the random vector is used to circumvent this obstacle.
This idea leads to the \emph{annealed} Langevin dynamics~\citep{annealed_langevin, kawar2021snips}.

\subsection{Annealed Langevin dynamics} \label{subsec:annealed_langevin}
To simplify the notation of what follows, we use $\bbA$ and its half-vectorization $\bba = \vech{\bbA}$ interchangeably.
Consider a noisy version of $\bba$,
\begin{equation}\label{eq:pert_symbs}
    \tba = \bba + \bbv,
\end{equation}
where $\bbv$ represents additive Gaussian noise.
More precisely, let $\{\sigma_l\}_{l=1}^{L}$ be a sequence of \emph{noise levels} such that $\sigma_1 > \sigma_2 > \cdots > \sigma_L > 0$. 
Then, for each noise level we define $\bbv_l \sim \ccalN(\pmb{0}, \sigma_l^2 \bbI)$.
In this setting, $\tba$ is continuous, and the iterative procedure involving annealed Langevin dynamics for our problem is given by
\begin{equation} \label{eq:annealed_langevin}
	\tba_{t+1} = \tba_t + \alpha_t \nabla_{\tba}\log p(\tba_t \mid \bbX) + \sqrt{2\alpha_t}\, \bbz_t,
\end{equation}
where $\alpha_t = \epsilon \cdot \sigma_{l(t)}^2 / \sigma_L^2$ and $l(t)$ is an increasing function mapping time steps $t$ to the annealing noise levels $l$.
Note that the noise present in $\tba_t$ (i.e., the variance of $\bbv_{l(t)}$) decreases with $t$, as given by the varying step size $\alpha_t$.

The annealed version of the dynamics was initially introduced to allow the algorithm to converge faster and perform better~\citep{annealed_langevin}.
However, in our case, it also offers the advantage of rendering the problem differentiable.
Consequently, the annealing enables the computation of the score functions and the use of Langevin dynamics to sample from an originally discrete distribution.
If the noise levels $\{\sigma_l\}_{l=1}^{L}$ and the step size $\epsilon$ are chosen adequately~\citep{annealed_langevin}, after a sufficiently large number of iterations, the sample $\tba_t$ is arbitrarily close to an actual sample from the discrete distribution $p(\bba \mid \bbX)$.
If an actual sample is needed, the noisy sample $\tba_t$ must be projected onto the set $\{0, 1\}$.

Now we need to compute the annealed score $\nabla_{\tba}\log p(\tba \mid \bbX)$ to sample graphs using~\eqref{eq:annealed_langevin}.
To avoid the use of cumbersome notation in what follows, from now on, we drop the reference to $\tba$ in the gradients, as we always take the derivatives with respect to that vector.
Using~\eqref{eq:approx_posterior}, we express the (approximate) annealed posterior score as
\begin{equation}  \label{eq:score_posterior_annealed}
    \nabla \log \hhatp(\tbA \mid \bbX) =
    \nabla \log \ccalL_\bbX(\tbTheta)  + \nabla \log p(\tbA),
\end{equation}
where we have defined
\begin{equation}
    \tbTheta = \hbTheta \circ (\tbA + \bbI).
\end{equation}
We next discuss each of the two terms in \eqref{eq:score_posterior_annealed}. Starting with $\nabla \log \ccalL_\bbX(\tbTheta)$, referred to as the \textit{annealed likelihood score}, we compute it as~[cf.~\eqref{eq:likelihood_theta}]
\begin{equation} \label{eq:score_likelihood_annealed}
    \nabla \log \ccalL_\bbX(\tbTheta) = \frac{k}{2} \nabla \log\det(\tbTheta) - \frac{k}{2} \nabla \tr(\bbS \tbTheta),
\end{equation}
with the two gradients in~\eqref{eq:score_likelihood_annealed} being straightforward to compute~\citep{cookbook}. Specifically, let us define $\Delta\tbSigma = \tbTheta^{-1} - \bbS$ and use $\bbT^{ij}$ to denote a matrix whose entries are all equal to zero except the $(i, j)$-th and the $(j, i)$-th ones, which are one.
Then,
\begin{equation} \label{eq:score_likelihood}
    \der{\log \ccalL_\bbX(\tbTheta)}{\tdA_{ij}}
    =
    \frac{k}{2}\tr\left[\left(2\Delta\tbSigma + \Delta\tbSigma \circ\bbI \right)
    \left(\hbTheta \circ \bbT^{ij} \right) 
    \right].
\end{equation}

We shift now to $\nabla \log p(\tbA)$, the second term in \eqref{eq:score_posterior_annealed}, which is referred to as the \textit{annealed prior score} and is more difficult to obtain. 
Note that computing $p(\tbA)$ requires convolving $p(\bbA)$ with the distribution of the noise [cf.~\eqref{eq:pert_symbs}], which is infeasible not only because of the computational burden of that task but also because we do not know $p(\bbA)$.
The alternative that we propose is to \emph{estimate} the annealed prior score $\nabla \log p(\tbA)$ just using samples from the prior $p(\bbA)$ (i.e., the available dataset $\ccalA$), as in~\cite{sevilla2023bayesian}.
We model this estimate as a GNN, where weights are trained on the dataset $\ccalA$, as we explain in Section~\ref{subsec:gnn}.

\subsection{Learned annealed scores} \label{subsec:gnn}
Let $\bbg_{\bbxi}(\tba, \sigma)$ be the output of the GNN we wish to train, with $\bbxi$ being its trainable parameters.
Ideally, the output for a given $\tba$ (with the associated noise level $\sigma_l$ of the current iteration) should be as close as possible to the actual score $\nabla \log p(\tba)$.
The loss function to learn $\bbxi$ should be designed to jointly minimize the mean squared error across all noise levels.
To achieve this, we define the distance
\begin{equation}
\begin{split}
    \ccalD\left(\tba\mid\bbxi, \sigma_l\right)
    &=
    \norm{\bbg_{\bbxi} (\tba, \sigma_l) - \nabla \log p(\tba\mid\bba)}_2^2 \\
    &=
    \norm{\bbg_{\bbxi} (\tba, \sigma_l) - (\bba - \tba)/\sigma_l^{2}}_2^2
\end{split}
\end{equation}
and the associated loss function
\begin{equation}
    \ccalJ\left(\bbxi\mid\{\sigma_l\}_{l=1}^L\right)
    =
    \frac{1}{2L}\sum_{l=1}^L \sigma_l^2 
    \E{\ccalD\left(\tba\mid\bbxi, \sigma_l\right)}.
    \label{eq:loss_cond}
\end{equation}
Following the proof in~\citet{scorematching_autoencoders}, it follows that the output of a GNN trained with~\eqref{eq:loss_cond} correctly estimates $\nabla \log p(\tba)$.
It is worth pointing out that the term $(\bba - \tba)/\sigma_l^2$ is known during training: $\bba$ is one element of $\ccalA$ and both $\tba$ and $\sigma_l$ are the GNN inputs.

The architecture of the GNN must account for the fact that the same graph can be represented by different adjacency matrices, depending on the node labeling.
In this work, we leverage the EDP-GNN~\citep{edpgnn}, designed to perform score-matching on graphs by proposing a permutation equivariant method to model the score function of interest.

\subsection{Final algorithm} \label{subsec:algorithm}
Now we need to put all the pieces together: the proposed (consistent) estimator $\hbA$ (Section~\ref{subsec:estimator}), the Langevin dynamics to get the samples to compute that estimator (Sections~\ref{subsec:langevin} and~\ref{subsec:annealed_langevin}), and the GNN training to estimate the score needed to run the Langevin dynamics (Section~\ref{subsec:gnn}).
The final scheme is described in Algorithm~\ref{alg:annealed_langevin}.
Notice that the score estimator $\bbg_{\bbxi}(\cdot)$ is an input.
Namely, before performing any GGM estimation, a GNN has to be trained with the desired dataset $\ccalA$ in order to be able to compute $\bbg_{\bbxi}\left(\tba, \sigma\right) \simeq \nabla \log p\left(\tba\right)$ for the different noise levels $\sigma_l$.

The first step in Algorithm~\ref{alg:annealed_langevin} is to compute $\hbTheta$ as in~\eqref{eq:glasso}. 
We then draw samples from the approximate posterior distribution $\hat{p}(\bbA \mid \bbX)$ by running the dynamics in~\eqref{eq:annealed_langevin}.
Recall that this is possible because a) we count on a closed-form (approximate) expression for the annealed likelihood~\eqref{eq:score_likelihood}, and b) we have found a way to estimate the annealed prior score by training a GNN with $\ccalA$.
Notice that, in each step, we just update the values of $\tbA^\ccalU$, leaving the known values in $\tbA^\ccalO$ fixed.

After $LT$ steps for each sample, the algorithm generates a continuous matrix $\tbA$.
As we work with unweighted graphs, it is necessary to make the prediction binary-valued.
Therefore, the algorithm draws $\ind{\tbA \geq 0.5}$ as a sample instead, representing an element-wise projection onto the set $\{0, 1\}$.
Following this procedure, we draw $M$ samples and then compute their average.
Lastly, we apply a threshold $\tau_k$ to the approximate posterior mean to compute the consistent estimator $\hbA$.

\begin{algorithm}[t]
    \caption{Annealed Langevin for GGM estimation}\label{alg:annealed_langevin}
    \begin{algorithmic}[1]
        \Require $\bbX, \bbA_0^\ccalO, \bbg_{\bbxi}(\cdot), \{\sigma_l\}_{l=1}^L, M, T, \epsilon, \tau_k$
        \State $\bbS \leftarrow \frac{1}{k} \bbX \bbX^\top$
        \State Compute $\hbTheta$ as in~\eqref{eq:glasso} \label{lst:line:theta}
        \State $\ccalS \leftarrow \{ \}$ \algorithmiccomment{Set of generated samples}
        \Repeat
            \State Initialize $\tbA_0^\ccalU \sim \ccalN(0.5, 0.5\bbI)$
            \State $\tbA_0^\ccalO \leftarrow \bbA_0^\ccalO$ \algorithmiccomment{Fix the known values}
            \For{$l \leftarrow 1\; \text{to}\;  L$}
                \State $\alpha_l \leftarrow \epsilon \cdot \sigma_l^2 / \sigma_L^2$ \algorithmiccomment{Change the noise level}
                \For{$t \leftarrow 1\; \text{to}\; T$}
                    \State Draw $\bbZ_t \sim \ccalN(\pmb{0}, \bbI)$
                    
                    \State Compute $\nabla \log \ccalL_\bbX (\tbTheta_t)$ as in~\eqref{eq:score_likelihood} \label{lst:line:score_likelihood}
                    \State Compute $\bbg_{\bbxi}(\tbA_{t-1}, \sigma_l)$ \label{lst:line:score_prior}
                    \State $\Delta_t \leftarrow \nabla \log \ccalL_\bbX (\tbTheta_t) + \bbg_{\bbxi}(\tbA_{t-1}, \sigma_l)$ \label{lst:line:score}
                    \State $\tbA_t^\ccalU \leftarrow \tbA_{t-1}^\ccalU + \alpha_l \Delta_t^\ccalU + \sqrt{2\alpha_l}\bbZ_t$
                    \State $\tbA_t^\ccalO \leftarrow \tbA_{t-1}^\ccalO$
                \EndFor
                \State $\tbA_0 \leftarrow \tbA_T$
            \EndFor
            \State $\tbA \leftarrow \tbA_T$ \algorithmiccomment{A sample from $p(\tbA \mid \bbX)$}
            \State $\bbA \leftarrow \ind{\tbA \geq 0.5}$ \algorithmiccomment{Project onto $\{ 0, 1 \}$}
            \State $\ccalS \leftarrow \ccalS \bigcup \{\bbA \}$
        \Until $\ccalS$ contains $M$ samples 
        \State Store the sample mean of $\ccalS$ in $\bbA_{\mathrm{mean}}$ \\
        $\hbA \leftarrow \ind{\bbA_{\mathrm{mean}} \geq \tau_k}$ \label{lst:line:threshold} \\
        \Return $\hbA$
    \end{algorithmic}
\end{algorithm}

\section{NUMERICAL RESULTS} \label{sec:num_results}
We carry out simulations in different setups~\footnote{Source code is available at \url{https://github.com/Tenceto/langevin_ggm}.} to demonstrate our scheme's practical relevance and gain insight regarding how informative the prior knowledge is when estimating $\bbA_0$.

In all the simulations, we first generate a fully-known graph and then drop $|\ccalU|$ random entries of $\bba_0$, which we then try to estimate.
We generate $M=10$ samples for each graph to compute~\eqref{eq:a_est}.
We compare our method with:~\footnote{Additional details on hyperparameter choices, properties of datasets, and computation of the reported metrics can be found in Section B of the SM.}
\begin{itemize}[noitemsep, labelindent=0pt, labelwidth=!, wide]
    \item \textbf{WGL}~\citep{weighted_glasso_1}. We penalize the indices in $\ccalU$ with a parameter $\lambda$, use an arbitrarily large penalty where $\bbA_0^\ccalO$ is $0$, and do not penalize the entries where $\bbA_0^\ccalO$ is $1$.
    \item \textbf{Thresholding}. We compute $\hbTheta$ and threshold it.
    \item \textbf{TIGER}. The GGM estimation method in~\citet{tiger} which does not require tuning.
    \item \textbf{GraphSAGE}. A link prediction method (not a GGM estimation method like the others) based on GNNs~\citep{graphsage}. 
    We use the measurements $\bbX$ as node features and $\bbA_0^\ccalO$ as the training set while testing $\bbA_0^\ccalU$.
\end{itemize}
All the thresholds ($\tau_k$ for our algorithm and those used for the thresholding and GraphSAGE methods) and $\lambda$ for WGL are tuned using a training set.
It is worth pointing out that the information given by $\bbA_0^\ccalO$ cannot be used within the TIGER algorithm, as it requires fixing some entries of $\bbTheta$.

Additionally, we use as a benchmark a variant of Algorithm~\ref{alg:annealed_langevin}.
We label it as ``Langevin prior'' (LPr), since it consists of just using prior information (i.e., $\Delta_t = \bbg_{\bbxi}(\tbA_{t-1}, \sigma_l)$ in line~\ref{lst:line:score}).
In other words, we test the algorithm when no observations are available, but only $\ccalA$ is.
Our method is labeled as ``Langevin posterior'' (LPost), using both the prior and likelihood score functions.

We run simulations for three different kinds of graphs.
Two of them, grid graphs and Barabási–Albert graphs~\citep{barabasi_albert}, are synthetic, while the third one consists of ego-nets of Eastern European users collected from the music streaming service Deezer~\citep{deezer}.
Next, we report and discuss the numerical results for all simulated scenarios.
We report the average F1 score over $10$ different train/test splits over $100$ graphs in each case.

\vspace{1mm}
\noindent {\bf Partially unknown grids.} 
We consider grids of different heights and widths with few additional random edges.
Results are shown in Figure~\ref{fig:grids}.

As $k$ increases, the performance of the predictors that use $\bbX$ increases, except for GraphSAGE.
Recall that the presence of an edge between two nodes does not imply a direct correlation between the variables, but rather \emph{conditional dependence} given the rest of the graph.
Considering that this relationship is not captured by local neighborhoods, which is how GraphSAGE aggregates node data, this method is expected to not benefit from including more observations.

For both sizes of $|\ccalU|$ and for the four different ratios $k / |\ccalU|$, LPost outperforms all the other approaches.
However, it is worth pointing out that the gap is much more prominent in Figure~\ref{fig:grids_10} when $|\ccalU|$ is smaller.
LPr's predictions also present a higher F1 score in that case.
This behavior leads to thinking that the information provided by $\ccalA$ decreases as $|\ccalU|$ (and, thus, the dimensions of the space from which the Langevin process is sampling) increases.
A complementary experiment on the performance dependence on $|\ccalA|$ for a different type of graph is presented in Section C.1 of the SM.

When $|\ccalU|$ is small, the prior probability mass is concentrated among fewer possible graphs.
Intuitively, in this case, Langevin generally samples either the same graph or similar ones throughout the different $M$ samples.
Thus, the sample mean yields a satisfactory estimate.
When $|\ccalU|$ is large, the probability mass is spread across many adjacency matrices in the high-dimensional space of $p(\tbA)$.
This leads to Langevin converging to diverse graphs each time we sample, reducing the usefulness of the sample mean as an estimator.
An additional experiment illustrating the performance dependence on $|\ccalU|$ is presented in Section~\ref{sec:performance_u} of the SM.

\begin{figure}[t]
    \centering
    \begin{subfigure}{\figonecol\columnwidth}
        \includegraphics[width=\textwidth]{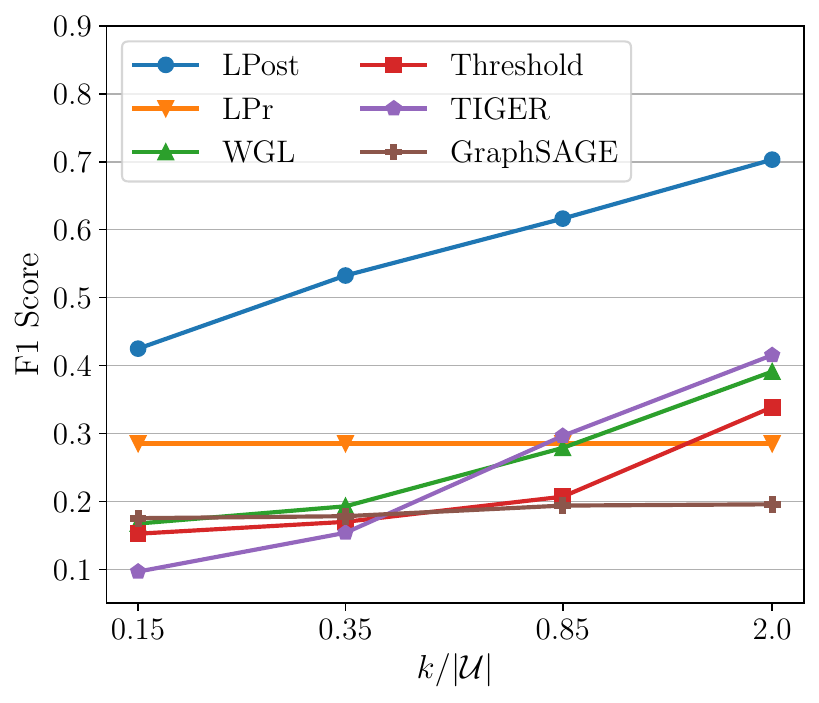}
        \caption{}
        \label{fig:grids_10}
    \end{subfigure}

    \begin{subfigure}{\figonecol\columnwidth}
        \includegraphics[width=\textwidth]{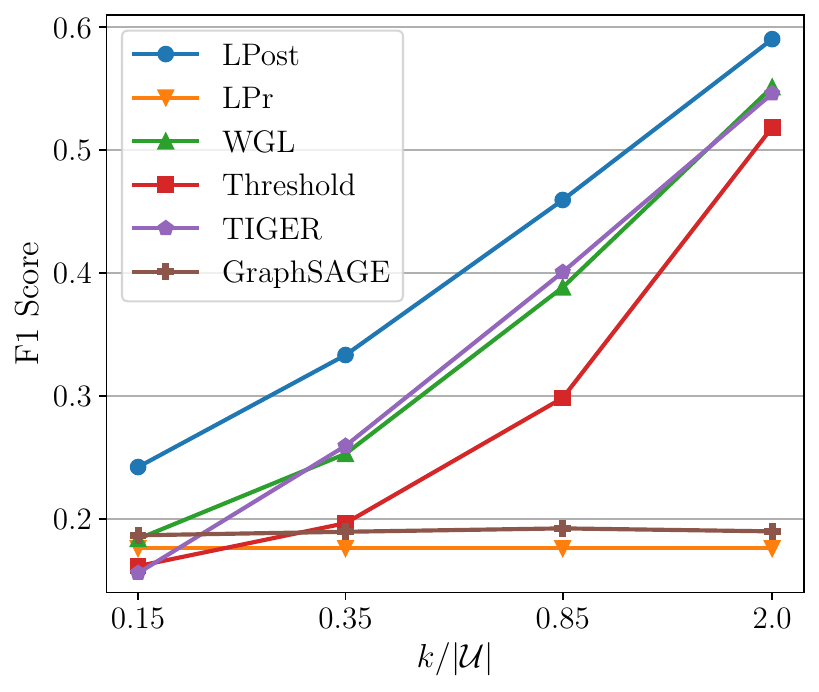}
        \caption{}
        \label{fig:grids_20}
    \end{subfigure}
    \caption{
    F1 score of several methods using grid graphs with $40 \! \leq \! n \! \leq \! 50$ where (a) $10\%$ and (b) $20\%$ of the values in $\bba$ are unknown.
    }
    \label{fig:grids}
\end{figure}

\vspace{1mm}
\noindent {\bf Partially unknown Barabási–Albert graphs.} 
Now we consider the dual Barabási–Albert preferential attachment model~\citep{dual_barabasi_albert}.
All graphs in $\ccalA$ are such that $n \in \{47, 49, 51, 53\}$, while we used graphs with $n \in \{46, 48, 50, 52\}$ nodes to test the algorithm.
This allows us to verify whether the EDP-GNN correctly generalizes the score estimation.
The results are shown in Figure~\ref{fig:barabasi}.

Once again, LPost yields better results than the other algorithms, mainly when $k$ is small.
As more observations are available, all of the methods (except for LPr and GraphSAGE) have approximately the same performance -- the information provided by 
$\ccalA$ becomes negligible compared to that offered by $\bbX$.

The F1 score achieved by LPr is relatively poor due to the large randomness in the graphs.
Namely, it is always worse than the one obtained using WGL.
On the contrary, when the underlying graph presents more structure (for instance, the grid graphs in Figure~\ref{fig:grids_10}), LPr was shown to outperform WGL for some values of $k / |\ccalU|$.
We can conclude that some priors offer more predictive power than others: the more substantial the structure of the graphs, the more useful $p(\bbA)$ becomes.

\vspace{1mm}
\noindent {\bf Partially unknown ego-nets.} 
Now, we consider the graphs in the Deezer dataset with $n \leq 25$.
The results are shown in Figure~\ref{fig:egonets}.

Once again, our method exhibits a higher edge prediction performance than the rest.
The behavior is similar to the one observed in the previous setups: the accuracy of all GGM-based methods increases with $k$.
GraphSAGE slightly outperforms LPost in this scenario for the smallest values of $k$.
Ego-nets are strongly local-based, and GraphSAGE is expected to outperform the rest of the approaches when the information provided by the observations is negligible.

Ego-nets, like grids, present a strong structure, rendering the prior highly predictive.
Even though half of the graph is unknown, the F1 score of LPr is the highest among all the experiments (cf. Figures~\ref{fig:grids} and~\ref{fig:barabasi}).

\begin{figure}[t]
    \centering
    \begin{subfigure}{\figonecol\columnwidth}
        \includegraphics[width=\textwidth]{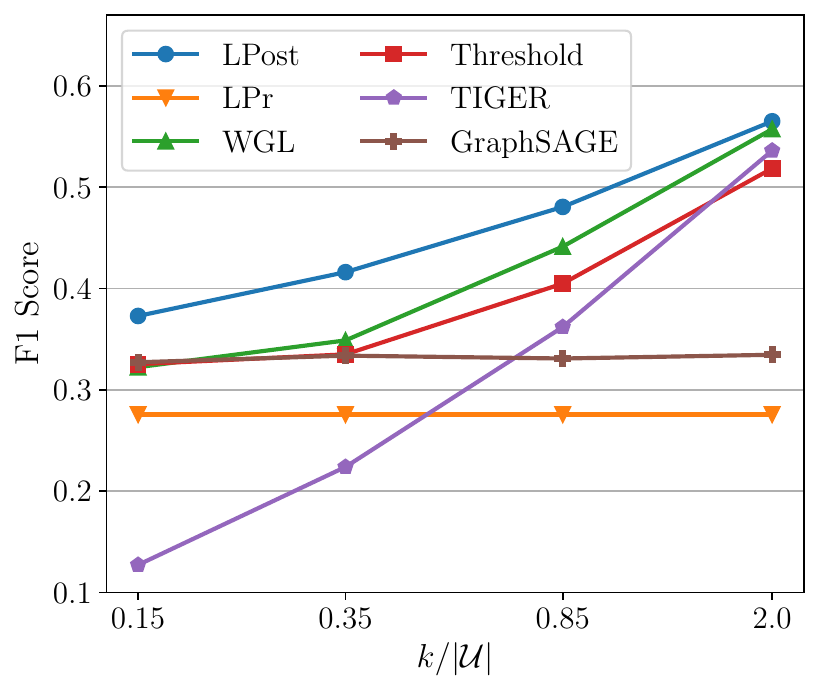}
        \caption{}
        \label{fig:barabasi}
    \end{subfigure}

    \begin{subfigure}{\figonecol\columnwidth}
        \includegraphics[width=\textwidth]{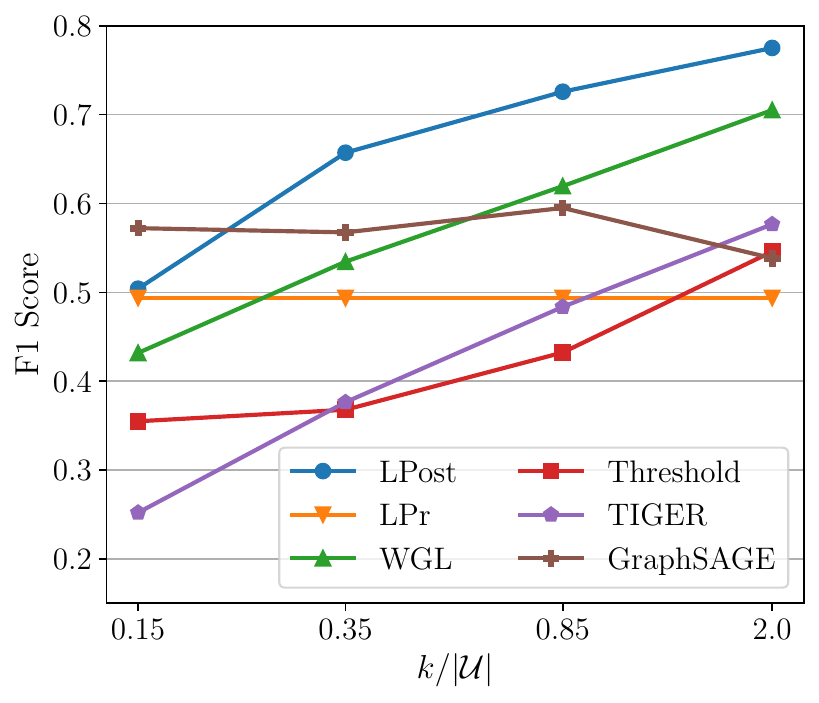}
        \caption{}
        \label{fig:egonets}
    \end{subfigure}
    \caption{
    F1 score of several methods using (a)~Barabási-Albert graphs with $|\ccalU| = 0.1\dim(\bba)$, and (b)~ego-nets with $|\ccalU| = 0.5\dim(\bba)$.
    }
    \label{fig:barabasi_and_ego}
\end{figure}

\vspace{1mm}
\noindent {\bf Fully unknown ego-nets.}
We consider the same dataset as in the last experiment, but now all the entries in $\bbA$ (except for those in the diagonal, which are $0$) are assumed unknown so that $|\ccalO| = 0$.
As shown in Figure~\ref{fig:grids}, the prior offers less predictive power as $|\ccalU|$ increases (see Section~\ref{sec:performance_u} for another experiment investigating this behavior).
Thus, when the entire graph is to be estimated, we propose fixing some of the entries in $\bbA$ with a graphical version of the \emph{random lasso}~\citep{wang2011random}.

To that end, we first compute the GL solution $B$ times (we use $B=50$ in these experiments), where a different bootstrap sample $\bbX^{(b)}$ is used for each iteration to obtain $\hbTheta^{(b)}_\mathrm{boot}$.
Then, the average 
\begin{equation} \label{eq:bootstrap}
    \hbA_\mathrm{boot} = \frac{1}{B}\sum_{b=1}^B \ind{\hbTheta^{(b)}_\mathrm{boot} \neq 0}
\end{equation}
can be interpreted as the probability of each $A_{ij}$ to be $1$.
Thus, for some probability margin $p_m$, we assume known some entries $A_{ij}$ such that
\begin{equation}
    A^{\ccalO}_{ij} =
    \begin{cases} 
        0 & \text{if } \hhatA_{\mathrm{boot}_{ij}} < 0.5 - p_m \\
        1 & \text{if } \hhatA_{\mathrm{boot}_{ij}} > 0.5 + p_m
    \end{cases},
\end{equation}
leaving as unknown all entries $(i,j)$ such that $0.5 - p_m \leq \hhatA_{\mathrm{boot}_{ij}} \leq 0.5 + p_m$.
As $p_m$ increases, the confidence of the estimated fixed values is higher, and $|\ccalO|$ becomes smaller.

The results for two different values of $p_m$ are shown in Figure~\ref{fig:full_egonets}, where we also compare with GL and its bootstrapped counterpart (BGL).
The latter is computed as in~\eqref{eq:bootstrap} and then thresholded.
We observe that a naive implementation of our method falls behind when $|\ccalO| = 0$, an expected behavior as analyzed in the experiments with grid graphs.
However, by leveraging the bootstrapping procedure to fix some entries in $\bbA$ we outperform GL and BGL, indicating that the prior distribution significantly contributes to the prediction accuracy.

Overall, our numerical experiments show that i) our approach leads to better graph estimation results than the classical alternatives considered and ii) the benefits of our approach are more significant when the number of observations is small and the graph presents marked structural features.


\begin{figure}[t]
    \centering
    \includegraphics[width=\figonecol\columnwidth]{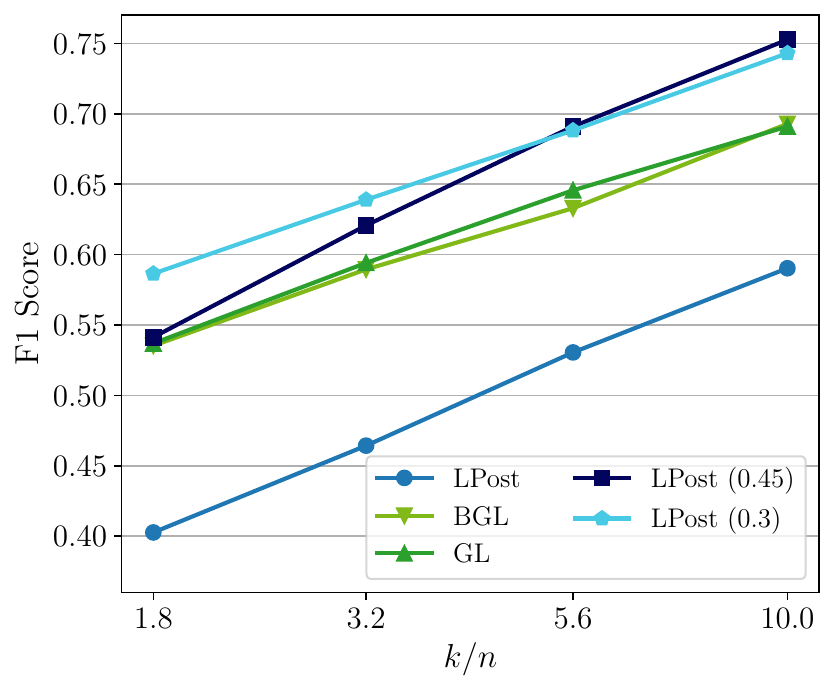}
    \caption{
    F1 score comparison when estimating ego-nets with no known values in $\bbA$.
    The values in parentheses correspond to the $p_m$ used to fix values from $\hbA_\mathrm{boot}$ prior to using our method LPost.
    }
    \label{fig:full_egonets}
\end{figure}


\section{CONCLUSIONS}
We proposed a GGM estimation algorithm based on annealed Langevin dynamics that allows us to leverage graph structural priors beyond sparsity.
Our approach exploits a set of known graphs to extract the prior distribution. 
We designed an algorithm that, by combining annealed Langevin dynamics with a GNN-based annealed prior score estimator, was able to draw samples from the posterior distribution of interest, namely the distribution of the unknown edges given the known ones, the structural prior, and the GMRF observations.
Finally, we proposed a consistent point estimate for the graph that underlies the GGM based on the sample posterior mean.
Through numerical experiments, we showed our method outperforms classical ones, especially in cases with few observations and highly structured graphs.


\subsubsection*{Acknowledgments}
This research was sponsored by the Army Research Office under Grant Number W911NF-17-S-0002; the Spanish (MCIN/AEI/10.13039/501100011033) Grants PID2019-105032GB-I00 and PID2022-136887NB-I00; the Autonomous Community of Madrid within the ELLIS Unit Madrid framework; and the Fulbright U.S. Student Program, in turn sponsored by the U.S. Department of State and the U.S.--Argentina Fulbright Commission.
The views and conclusions contained in this document are those of the authors and should not be interpreted as representing the official policies, either expressed or implied, of the Army Research Office, the U.S. Army, the Fulbright Program, the U.S.--Argentina Fulbright Commission, or the U.S. Government. 
The U.S. Government is authorized to reproduce and distribute reprints for Government purposes, notwithstanding any copyright notation herein.

\bibliography{citations}

\begin{thebibliography}{}

\bibitem[Banerjee et~al., 2008]{glasso_banerjee}
Banerjee, O., El~Ghaoui, L., and d'Aspremont, A. (2008).
\newblock Model selection through sparse maximum likelihood estimation for
  multivariate {G}aussian or binary data.
\newblock {\em The Journal of Machine Learning Research}, 9:485--516.

\bibitem[Barabási and Albert, 1999]{barabasi_albert}
Barabási, A.-L. and Albert, R. (1999).
\newblock Emergence of scaling in random networks.
\newblock {\em Science}, 286(5439):509--512.

\bibitem[Bishop and Nasrabadi, 2006]{bishop}
Bishop, C.~M. and Nasrabadi, N.~M. (2006).
\newblock {\em Pattern Recognition and Machine Learning}.
\newblock Springer.

\bibitem[Casella and Berger, 2021]{casella2021statistical}
Casella, G. and Berger, R.~L. (2021).
\newblock {\em Statistical Inference}.
\newblock Cengage Learning.

\bibitem[Codazzi et~al., 2022]{CODAZZI2022107416}
Codazzi, L., Colombi, A., Gianella, M., Argiento, R., Paci, L., and Pini, A.
  (2022).
\newblock {G}aussian graphical modeling for spectrometric data analysis.
\newblock {\em Computational Statistics \& Data Analysis}, 174:107416.

\bibitem[Dempster, 1972]{dempster}
Dempster, A.~P. (1972).
\newblock Covariance selection.
\newblock {\em Biometrics}, pages 157--175.

\bibitem[Dobra et~al., 2010]{dobra2010modeling}
Dobra, A., Eicher, T.~S., and Lenkoski, A. (2010).
\newblock Modeling uncertainty in macroeconomic growth determinants using
  {G}aussian graphical models.
\newblock {\em Statistical Methodology}, 7(3):292--306.

\bibitem[Dobra et~al., 2004]{dobra2004sparse}
Dobra, A., Hans, C., Jones, B., Nevins, J.~R., Yao, G., and West, M. (2004).
\newblock Sparse graphical models for exploring gene expression data.
\newblock {\em Journal of Multivariate Analysis}, 90(1):196--212.

\bibitem[Friedman et~al., 2008]{glasso}
Friedman, J., Hastie, T., and Tibshirani, R. (2008).
\newblock Sparse inverse covariance estimation with the graphical lasso.
\newblock {\em Biostatistics}, 9(3):432--441.

\bibitem[Friedman and Koller, 2003]{friedman2003being}
Friedman, N. and Koller, D. (2003).
\newblock Being {B}ayesian about network structure. {A} {B}ayesian approach to
  structure discovery in {B}ayesian networks.
\newblock {\em Machine learning}, 50:95--125.

\bibitem[Grzebyk et~al., 2004]{grzebyk2004identification}
Grzebyk, M., Wild, P., and Chouani{\`e}re, D. (2004).
\newblock On identification of multi-factor models with correlated residuals.
\newblock {\em Biometrika}, 91(1):141--151.

\bibitem[Hamilton et~al., 2017]{graphsage}
Hamilton, W., Ying, Z., and Leskovec, J. (2017).
\newblock Inductive representation learning on large graphs.
\newblock {\em Advances in Neural Information Processing systems}, 30.

\bibitem[Hosseini and Lee, 2016]{grab}
Hosseini, M.~J. and Lee, S.-I. (2016).
\newblock Learning sparse {G}aussian graphical models with overlapping blocks.
\newblock {\em Advances in neural information processing systems}, 29.

\bibitem[Hunter and Handcock, 2006]{ergm_def_hunter}
Hunter, D.~R. and Handcock, M.~S. (2006).
\newblock Inference in curved exponential family models for networks.
\newblock {\em Journal of Computational and Graphical Statistics},
  15(3):565--583.

\bibitem[Kawar et~al., 2021]{kawar2021snips}
Kawar, B., Vaksman, G., and Elad, M. (2021).
\newblock {SNIPS}: Solving noisy inverse problems stochastically.
\newblock {\em Advances in Neural Information Processing Systems},
  34:21757--21769.

\bibitem[Li et~al., 2020]{li2020high}
Li, T., Qian, C., Levina, E., and Zhu, J. (2020).
\newblock High-dimensional {G}aussian graphical models on network-linked data.
\newblock {\em The Journal of Machine Learning Research}, 21(1):2851--2895.

\bibitem[Li and Jackson, 2015]{weighted_glasso_1}
Li, Y. and Jackson, S. (2015).
\newblock Gene network reconstruction by integration of biological prior
  knowledge.
\newblock {\em G3-Genes Genomes Genetics}, 5:1075--1079.

\bibitem[Liu and Wang, 2017]{tiger}
Liu, H. and Wang, L. (2017).
\newblock {TIGER}: A tuning-insensitive approach for optimally estimating
  {G}aussian graphical models.
\newblock {\em Electronic Journal of Statistics}, 11(1):241 -- 294.

\bibitem[Moshiri, 2018]{dual_barabasi_albert}
Moshiri, N. (2018).
\newblock The dual-{B}arab\'asi-{A}lbert model.

\bibitem[Niu et~al., 2020]{edpgnn}
Niu, C., Song, Y., Song, J., Zhao, S., Grover, A., and Ermon, S. (2020).
\newblock Permutation invariant graph generation via score-based generative
  modeling.
\newblock In {\em International Conference on Artificial Intelligence and
  Statistics}, pages 4474--4484. PMLR.

\bibitem[Petersen and Pedersen, 2012]{cookbook}
Petersen, K.~B. and Pedersen, M.~S. (2012).
\newblock The matrix cookbook.
\newblock Version 20121115.

\bibitem[Qiu and Liyanage, 2019]{thresholding}
Qiu, Y. and Liyanage, J.~S. (2019).
\newblock Threshold selection for covariance estimation.
\newblock {\em Biometrics}, 75(3):895--905.

\bibitem[Ravikumar et~al., 2011]{ravikumar_glasso}
Ravikumar, P., Wainwright, M.~J., Raskutti, G., and Yu, B. (2011).
\newblock {High-dimensional covariance estimation by minimizing l1-penalized
  log-determinant divergence}.
\newblock {\em Electronic Journal of Statistics}, 5(none):935 -- 980.

\bibitem[Robert and Casella, 1999]{MCMCbook}
Robert, C. and Casella, G. (1999).
\newblock {\em {M}onte {C}arlo Statistical Method}.
\newblock Springer.

\bibitem[Roberts and Tweedie, 1996]{Roberts1996ExponentialCO}
Roberts, G.~O. and Tweedie, R.~L. (1996).
\newblock Exponential convergence of {L}angevin distributions and their
  discrete approximations.
\newblock {\em Bernoulli}, 2:341--363.

\bibitem[Rozemberczki et~al., 2020]{deezer}
Rozemberczki, B., Kiss, O., and Sarkar, R. (2020).
\newblock {Karate Club: An API Oriented Open-source Python Framework for
  Unsupervised Learning on Graphs}.
\newblock In {\em ACM International Conference on Information and Knowledge
  Management}, page 3125–3132. ACM.

\bibitem[Rue and Held, 2005]{rue2005gaussian}
Rue, H. and Held, L. (2005).
\newblock {\em Gaussian Markov Random Fields: Theory and Applications}.
\newblock CRC press.

\bibitem[Sevilla and Segarra, 2023]{sevilla2023bayesian}
Sevilla, M. and Segarra, S. (2023).
\newblock Bayesian topology inference on partially known networks from
  input-output pairs.

\bibitem[Simpson and Laurienti, 2015]{SIMPSON2015310}
Simpson, S.~L. and Laurienti, P.~J. (2015).
\newblock A two-part mixed-effects modeling framework for analyzing whole-brain
  network data.
\newblock {\em NeuroImage}, 113:310--319.

\bibitem[Snijders et~al., 2006]{ergm_def_snijders}
Snijders, T. A.~B., Pattison, P.~E., Robins, G.~L., and Handcock, M.~S. (2006).
\newblock New specifications for exponential random graph models.
\newblock {\em Sociological Methodology}, 36(1):99--153.

\bibitem[Song and Ermon, 2019]{annealed_langevin}
Song, Y. and Ermon, S. (2019).
\newblock Generative modeling by estimating gradients of the data distribution.
\newblock {\em Advances in Neural Information Processing Systems}, 32.

\bibitem[Sundaram, 1996]{max_theorem}
Sundaram, R.~K. (1996).
\newblock {\em A first course in optimization theory}.
\newblock Cambridge University Press.

\bibitem[Tan et~al., 2017]{tan2017bayesian}
Tan, L. S.~L., Jasra, A., Iorio, M.~D., and Ebbels, T. M.~D. (2017).
\newblock {B}ayesian inference for multiple {G}aussian graphical models with
  application to metabolic association networks.
\newblock {\em The Annals of Applied Statistics}, 11(4):2222--2251.

\bibitem[Tsai et~al., 2022]{joint_ggm_survey}
Tsai, K., Koyejo, O., and Kolar, M. (2022).
\newblock Joint {G}aussian graphical model estimation: A survey.
\newblock {\em Wiley Interdisciplinary Reviews: Computational Statistics},
  14(6):e1582.

\bibitem[Vincent, 2011]{scorematching_autoencoders}
Vincent, P. (2011).
\newblock A connection between score matching and denoising autoencoders.
\newblock {\em Neural Computation}, 23(7):1661--1674.

\bibitem[Wang and Li, 2012]{gwishart_prior}
Wang, H. and Li, S.~Z. (2012).
\newblock Efficient {G}aussian graphical model determination under
  {G}-{W}ishart distributions.
\newblock {\em Electronic Journal of Statistics}, 6:168--198.

\bibitem[Wang et~al., 2011]{wang2011random}
Wang, S., Nan, B., Rosset, S., and Zhu, J. (2011).
\newblock Random lasso.
\newblock {\em The annals of applied statistics}, 5(1):468.

\bibitem[Wang et~al., 2020]{wang_ggm_2020}
Wang, Y., Segarra, S., and Uhler, C. (2020).
\newblock {High-dimensional joint estimation of multiple directed Gaussian
  graphical models}.
\newblock {\em Electronic Journal of Statistics}, 14(1):2439 -- 2483.

\bibitem[Welling and Teh, 2011]{wellinglang}
Welling, M. and Teh, Y.~W. (2011).
\newblock {B}ayesian learning via stochastic gradient {L}angevin dynamics.
\newblock In {\em Intl. Conf. on Machine Learning}, page 681–688.

\bibitem[Williams, 2020]{williams2020beyond}
Williams, D.~R. (2020).
\newblock Beyond lasso: A survey of nonconvex regularization in {G}aussian
  graphical models.

\bibitem[Wu et~al., 2019]{Wu852798}
Wu, Q., Zhang, Z., Waltz, J., Ma, T., Milton, D., and Chen, S. (2019).
\newblock Predicting latent links from incomplete network data using
  exponential random graph model with outcome misclassification.
\newblock {\em bioRxiv}.

\bibitem[Zhou et~al., 2021]{ggm_metabolomic}
Zhou, J., Hoen, A., Mcritchie, S., Pathmasiri, W., Viles, W., Nguyen, Q.,
  Madan, J., Dade, E., Karagas, M., and Gui, J. (2021).
\newblock Information enhanced model selection for {G}aussian graphical model
  with application to metabolomic data.
\newblock {\em Biostatistics}, 23.

\bibitem[Zhuang et~al., 2022]{weighted_glasso_3}
Zhuang, Y., Xing, F., Ghosh, D., Banaei-Kashani, F., Bowler, R.~P., and
  Kechris, K. (2022).
\newblock An augmented high-dimensional graphical lasso method to incorporate
  prior biological knowledge for global network learning.
\newblock {\em Frontiers in Genetics}, page 2405.

\bibitem[Zuo et~al., 2017]{weighted_glasso_2}
Zuo, Y., Cui, Y., Yu, G., Li, R., and Ressom, H. (2017).
\newblock Incorporating prior biological knowledge for network-based
  differential gene expression analysis using differentially weighted graphical
  lasso.
\newblock {\em BMC Bioinformatics}, 18.

\end{thebibliography}


\section*{Checklist}

\begin{enumerate}

 \item For all models and algorithms presented, check if you include:
 \begin{enumerate}
   \item A clear description of the mathematical setting, assumptions, algorithm, and/or model. \textbf{Yes}.
   \item An analysis of the properties and complexity (time, space, sample size) of any algorithm. \textbf{Yes}. Additional analysis is provided in the SM. In particular, for a complexity analysis please refer to Section~\ref{sec:implementation}.
   \item (Optional) Anonymized source code, with specification of all dependencies, including external libraries. \textbf{Yes}.
 \end{enumerate}

 \item For any theoretical claim, check if you include:
 \begin{enumerate}
   \item Statements of the full set of assumptions of all theoretical results. \textbf{Yes}.
   \item Complete proofs of all theoretical results. \textbf{Yes}. Proofs to the auxiliary lemmas are provided in the SM (Section~\ref{sec:proofs}).
   \item Clear explanations of any assumptions. \textbf{Yes}.     
 \end{enumerate}

 \item For all figures and tables that present empirical results, check if you include:
 \begin{enumerate}
   \item The code, data, and instructions needed to reproduce the main experimental results (either in the supplemental material or as a URL). \textbf{Yes}. Please refer to the source code if needed.
   \item All the training details (e.g., data splits, hyperparameters, how they were chosen). \textbf{Yes}. Please refer to Sections~\ref{sec:hyperparams} and~\ref{sec:details} of the SM for more information.
    \item A clear definition of the specific measure or statistics and error bars (e.g., with respect to the random seed after running experiments multiple times). \textbf{Yes}. Additional information can be found in Section~\ref{sec:details} of the SM.
    \item A description of the computing infrastructure used. (e.g., type of GPUs, internal cluster, or cloud provider). \textbf{Yes}. See Section~\ref{sec:implementation} in the SM.
 \end{enumerate}

 \item If you are using existing assets (e.g., code, data, models) or curating/releasing new assets, check if you include:
 \begin{enumerate}
   \item Citations of the creator if your work uses existing assets. \textbf{Yes.}
   \item The license information of the assets, if applicable. \textbf{Yes}. Publication of our source code is available in GitHub under an MIT License.
   \item New assets either in the supplemental material or as a URL, if applicable. \textbf{Yes}.
   \item Information about consent from data providers/curators. \textbf{Not Applicable}.
   \item Discussion of sensible content if applicable, e.g., personally identifiable information or offensive content. \textbf{Not Applicable}.
 \end{enumerate}

 \item If you used crowdsourcing or conducted research with human subjects, check if you include:
 \begin{enumerate}
   \item The full text of instructions given to participants and screenshots. \textbf{Not Applicable}
   \item Descriptions of potential participant risks, with links to Institutional Review Board (IRB) approvals if applicable. \textbf{Not Applicable}
   \item The estimated hourly wage paid to participants and the total amount spent on participant compensation. \textbf{Not Applicable}
 \end{enumerate}

\end{enumerate}

\renewcommand{\thesection}{\Alph{section}}
\setcounter{section}{0}

%
\runningtitle{Estimation of partially known Gaussian graphical models with score-based structural priors}

%
\runningauthor{Martín Sevilla, Antonio G. Marques, Santiago Segarra}

\onecolumn
\aistatstitle{Estimation of partially known Gaussian graphical models with score-based structural priors \\
Supplementary Materials}

\vfilneg

\section{PROOFS OF LEMMAS} \label{sec:proofs}

\subsection{Proof of Lemma~\ref{lemma:theta}} \label{sec:proof_theta}
Let $\ccalS_{+}^n = \{ \bbV \in \reals^{n \times n} \mid \bbV \succeq 0\}$ be the set of all positive semidefinite matrices.
Then, we define a function $h: \ccalS_{+}^n \to \ccalS_{+}^n$ such that
\begin{align} \label{eq:h_def}
    h(\bbV) = 
    & \argmax_{\bbTheta \succeq 0}
    f(\bbTheta ; \bbV)  \nonumber \\
    & \text{s. to} \,\, \Theta_{ij} = 0 \,\,\,\,\,\, \forall (i,j): A^\ccalO_{ij}=0,
\end{align}
where $f: \ccalS_{+}^n \to \reals$ is $f(\bbTheta ; \bbV) = \log \det\bbTheta - \tr\left(\bbV\bbTheta \right)$.
It immediately follows that the estimator in \eqref{eq:glasso} of the main paper satisfies
\begin{equation} \label{eq:h_S}
    \hbTheta = h(\bbS),
\end{equation}
with $\bbS=\frac{1}{k}\bbX\bbX^\top$ being the sample covariance matrix. The estimator $\hbTheta$ is consistent if $h(\bbS)$ approaches $\bbTheta_0$ as $k$ increases.
Hence, this is what we want to prove next.

First, we compute $h\left(\bbTheta^{-1}_0\right)$.
To this end, we first check what matrix maximizes $f$ without considering the constraint in~\eqref{eq:h_def}.
The maximizer is unique since $f$ is continuous and strictly concave in $\ccalS_+^n$~\citep{ravikumar_glasso}.
Taking the gradient of $f$ with respect to $\bbTheta$ yields
\begin{equation}
    \der{f(\bbTheta ; \bbV)}{\bbTheta} = 2\bbTheta^{-1} - \bbTheta^{-1} \circ \bbI - 2\bbV - \bbV \circ \bbI =  \pmb{0} \iff \bbTheta = \bbV^{-1}.
\end{equation}
Hence, $f(\bbTheta ; \bbTheta_0^{-1})$ is maximized when $\bbTheta = \bbTheta_0$.
Furthermore, $\bbTheta_0$ satisfies the constraint in~\eqref{eq:h_def}.
As a result, it holds that 
\begin{equation} \label{eq:h_theta}
    h\left(\bbTheta^{-1}_0\right) = \bbTheta_0.
\end{equation}

Notice that the mapping $h$ is continuous.
This can be proven through the maximum theorem~\citep{max_theorem}.
Let $\ccalC = \{\bbV \in \reals^{n \times n} \mid V_{ij} = 0 \, \, \, \, \forall (i,j): A^\ccalO_{ij}=0 \}$ be the set of constrained matrices we are interested in.
Any linear combination of matrices in $\ccalC$ is still in $\ccalC$; thus, $\ccalC$ is a convex set.
Therefore, since $\ccalS_+^n$ is convex as well, the intersection $\ccalI = \ccalC \cap \ccalS_+^n$ (which is the set over which $f$ is maximized in~\eqref{eq:h_def} to compute $h$) is convex too.

The function $f$ is continuous and strictly concave in $\ccalI$, since $\ccalI \subseteq \ccalS_+^n$.
Hence, by the maximum theorem under convexity~\citep{max_theorem}, the $\argmax$ mapping of $f(\bbTheta; \bbV)$ within $\ccalI$ is a continuous function of $\bbV$.
Consequently, $h$ is a continuous mapping.

Consider that by the law of large numbers,
\begin{equation} \label{eq:consistency_s}
    \bbS = \frac{1}{k}\sum_{i=1}^k \bbx_i \bbx_i^\top
    \xrightarrow{k\to\infty}
    \E{\bbx \bbx^\top} = \bbTheta_0^{-1}.
\end{equation}
Consequently, by the continuous mapping theorem (which can be applied because $h$ is continuous), considering~\eqref{eq:consistency_s},~\eqref{eq:h_S} and~\eqref{eq:h_theta}, we conclude that
\begin{equation} \label{eq:consistency_theta}
    \hbTheta = h(\bbS)
    \xrightarrow{k\to\infty}
    h\left(\bbTheta_0^{-1}\right) = \bbTheta_0.
\end{equation}

\subsection{Proof of Lemma~\ref{lemma:p_hat}} \label{sec:proof_phat}

To study how $\hhatp(\bbA \mid \bbX)$ behaves as $k\to\infty$, we first analyze the likelihood function $\ccalL_\bbX\left(\bbTheta \right)$.
Using Bayes' theorem, we get
\begin{equation} \label{eq:posterior_theta}
    p(\bbTheta \mid \bbX)\propto \ccalL_\bbX\left(\bbTheta \right)p(\bbTheta).
\end{equation}
By the Bernstein--von Mises theorem, we know that $p(\bbTheta \mid \bbX) \xrightarrow{k\to\infty} \delta(\bbTheta - \bbTheta_0)$, where $\delta(\cdot)$ is the Dirac delta.
Combining this with~\eqref{eq:posterior_theta} it follows that
\begin{equation} \label{eq:likelihood_limit}
    \ccalL_\bbX\left(\bbTheta\right)
    \xrightarrow{k\to\infty}
    0 \quad \forall \bbTheta \neq \bbTheta_0.
\end{equation}

The result in~\eqref{eq:likelihood_limit} holds for any prior $p(\bbTheta)$ since the prior can be considered a constant with respect to $k$.

Given~\eqref{eq:likelihood_limit} and the consistency of $\hbTheta$ in Lemma~\ref{lemma:theta}, we conclude that [cf.~\eqref{eq:approx_posterior} in the main paper]
\begin{equation}
    \hhatp(\bbA \mid \bbX) 
    \xrightarrow{k\to\infty}
    \frac{p(\bbA)}{C} \delta \left(\bbTheta_0 \circ (\bbA + \bbI) - \bbTheta_0\right),
\end{equation}
where $C$ is a normalization constant.

\section{EXPERIMENTAL DETAILS} \label{sec:details}
\subsection{Hyperparameters} \label{sec:hyperparams}
Regarding the Langevin sampler, in all the experiments, we use $L=10$ noise levels, evenly spaced between $\sigma_1 = 0.5$ and $\sigma_L = 0.03$, and $T=300$ steps per level.
We set the step size at $\epsilon = 10^{-6}$.

To tune the regularization parameter $\lambda$ for the graphical lasso, we fit a function of the form $\lambda(k) = a\log(k)^2 + b\log(k) + c$ using a training set and then evaluate that function at inference time for the given value of $k$.
We selected this specific functional form as it showed a very satisfactory fit for all of our cases.
An example is shown in Figure~\ref{fig:tuning_glasso}.

\begin{figure}[ht]
     \centering
     \includegraphics[width=0.9\textwidth]{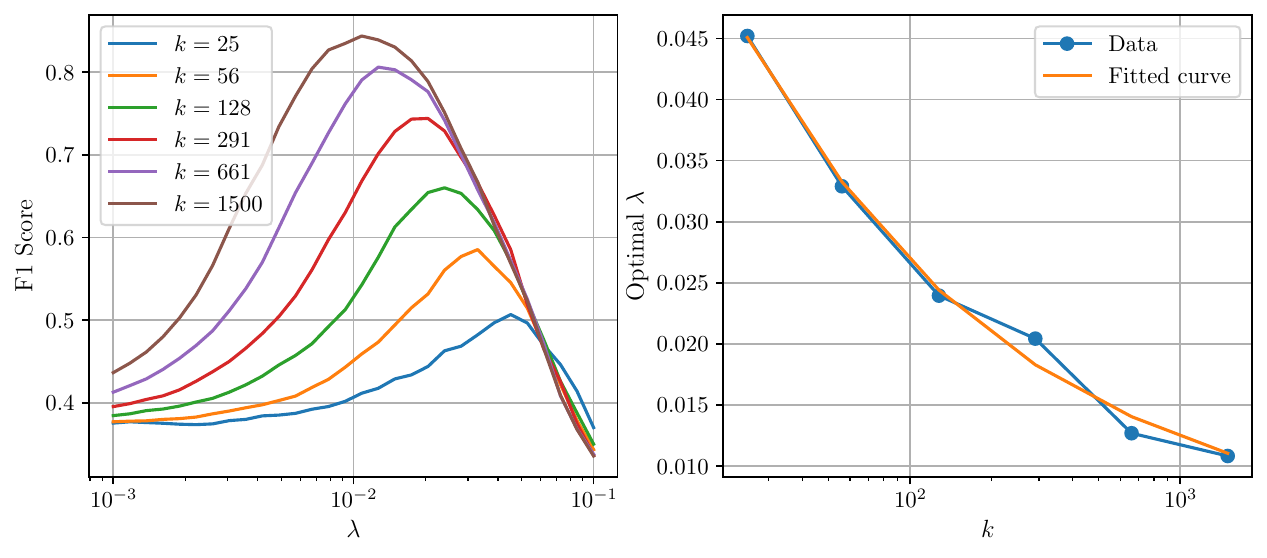}
     \caption{Fitting of $\lambda(k) = a\log(k)^2 + b\log(k) + c$ for the ego-nets dataset with $|\ccalU|=0.5\dim({\bba})$.
     The orange curve on the right is the one used at inference time.}
     \label{fig:tuning_glasso}
\end{figure}

The thresholds used to obtain a binary-valued estimate of $\bbA_0$ when using GraphSAGE, the thresholding of $\hbTheta$ and our method (in this last case, the threshold is $\tau_k$) are set using the same procedure.
Given $R$ simulations (in all of our experiments, $R=100$), we take $R/2$ of those continuous-valued matrices and find the threshold within some grid $\ccalT$ that works the best for each method.
Then, with those tuned values, we threshold the other $R/2$ matrices left and only evaluate performance over them.
All the plots shown in Section~\ref{sec:num_results} from the original paper, as well as in this document, correspond to averages following this procedure $S$ times (i.e., across $S=10$ different train/test splits in our case), each with its threshold.
Algorithm~\ref{alg:metrics} provides a more precise description of our method.

\begin{algorithm}[t]
    \caption{Metric computation}\label{alg:metrics}
    \begin{algorithmic}[1]
        \Require $\{\bba^\ccalU_r, \tba^\ccalU_r\}_{r=1}^R$, $\ccalT$, $S$
        \State $J \leftarrow 0$
        \For{$s \leftarrow 1\; \text{to}\; S$} \algorithmiccomment{Repeat the process with $S$ different splits}
            \State Store $R/2$ randomly chosen tuples $(\bba^\ccalU_r, \tba^\ccalU_r)$ into $\ccalA_{\mathrm{train}}$
            \State Store the other $R/2$ tuples in $\ccalA_{\mathrm{test}}$
            \For{$\tau \in \ccalT$} \algorithmiccomment{Get the best threshold in the grid}
                \State Compute $\mathrm{metric}(\bba^\ccalU_r, \ind{\tba^\ccalU_r \geq \tau})$ for each tuple in $\ccalA_{\mathrm{train}}$
                \State If the average train metric is the best so far, store the current threshold in $\tau^\star$
            \EndFor
            \State Compute $\mathrm{metric}(\bba^\ccalU_r, \ind{\tba^\ccalU_r \geq \tau^\star})$ for each tuple in $\ccalA_{\mathrm{test}}$
            \State Store the average of the test metrics in $J_s$
            \State $J \leftarrow J + J_s$
        \EndFor \\
        \Return $\frac{J}{S}$ \algorithmiccomment{Return the mean across splits}
    \end{algorithmic}
\end{algorithm}

\subsection{Datasets' details} \label{sec:datasets}
For the grids, we generated graphs with $40 \! \leq \! n \! \leq \! 50$.
Additionally, to introduce randomness into the graphs, we uniformly added between $2$ and $5$ edges at random to the edge set $\ccalE$.
We trained a GNN with $|\ccalA| = 5000$ graphs generated with the described procedure.

In the case of the dual Barabási–Albert model, the generation is initialized with an empty graph with $\max(n_1, n_2)$ nodes.
Then, the remaining $n-\max(n_1, n_2)$ vertices are added iteratively.
For each new node, $n_1$ edges are added with probability $\pi$, and $n_2$ edges are added with probability $1-\pi$.
Edges are added following a preferential attachment criterion.
In our case, we set $n_1=2$, $n_2=4$ and $\pi=0.5$.
We simulated $|\ccalA|=1000$ graphs in this setting.

Regarding the ego-nets, we used $|\ccalA|=2926$ of the graphs available in the whole dataset (i.e., only those such that $n \leq 25$) to train the EDP-GNN, leaving $100$ graphs for testing.

\subsection{Implementation details} \label{sec:implementation}
All experiments shown in the main paper and in this document and the training of the EDP-GNNs were run on an NVIDIA DGX A100 system.
Our code is provided as part of the supplementary material and not in a GitHub repository for the sake of anonymity.
The code corresponding to the EDP-GNN implementation was taken from the original repository of~\citet{edpgnn} under a GPL-3.0 license.

Regarding the computational complexity, the key steps to consider from Algorithm~\ref{alg:annealed_langevin} are those in lines~\ref{lst:line:theta}, \ref{lst:line:score_likelihood}, and~\ref{lst:line:score_prior}.
The computation time of $\hbTheta$ is $\ccalO(n^3)$ for dense problems, but much less in sparse ones~\citep{glasso}, and even less if some of the entries are fixed as is the case of~\eqref{eq:glasso}.
However, here, we consider the worst-case scenario.
After computing $\hbTheta$, we evaluate~\eqref{eq:score_likelihood}, which requires to compute the inverse of $\tbTheta$, leading to a time complexity of $\ccalO(n^3)$ as well.
Feed-forwarding the GNN consists of just matrix and vector multiplications that scale as $\ccalO(n^2)$.
Therefore, the dominating steps are the first two.
The step in~\ref{lst:line:score_likelihood} is repeated $L \cdot T$ times per sample, and since the sampling of each $\bbA_i$ is completely parallelizable, the total complexity of Algorithm~\ref{alg:annealed_langevin} is $\ccalO(L \cdot T \cdot n^3)$, considering that typically $LT \gg 1$.

\section{ADDITIONAL EXPERIMENTS} \label{sec:additional_exp}
This section provides more experiments to gain additional insights into our method.

\subsection{Performance dependence on \texorpdfstring{$|\ccalA|$}{|A|}}
We want to analyze how the predictive power of the prior learned by the EDP-GNN changes when different dataset sizes $|\ccalA|$ are used for training.
To this end, we use a different family of graphs, known as \textit{exponential random graph models} (ERGMs), which has a closed-form distribution (up to a normalization constant) that relies on a set of network statistics and parameters~\citep{ergm_def_hunter, ergm_def_snijders}.
Namely, for a vector of $r$ statistics $\bbpsi(\bbA)=\begin{bmatrix}\psi_1(\bbA) & \psi_2(\bbA) &\cdots& \psi_r(\bbA)\end{bmatrix}^\top$ and a set of parameters $\bbbeta \in \reals^r$, the distribution of $\bbA$ is given by
\begin{equation}
    p(\bbA) = \frac{1}{C_\psi(\bbbeta)} \exp \left( \bbbeta^\top \bbpsi(\bbA) \right) .
\end{equation}
The statistics can be, e.g., the number of triangles in the graph, the number of $d$-stars, or the number of edges, among many others. In the simulations, we consider an ERGM distribution with statistics
\begin{equation}
    \bbpsi(\bbA) = \begin{bmatrix}\mathrm{AKS}_\gamma (\bbA) & \frac12 \sum_{ij} A_{ij} \end{bmatrix} ^\top .
\end{equation}
The first one corresponds to the alternated $d$-stars statistic~\citep{ergm_def_hunter}, defined as
\begin{equation}
\mathrm{AKS}_\gamma (\bbA) = \sum_{d=2}^{p-1} (-1)^d \frac{S_d(\bbA)}{\gamma^{d-2}}, \ 
\end{equation}
with $S_d(\cdot)$ being the number of $d$-stars in the given graph and $\gamma$ a constant. 
In our simulations, we set $\gamma = 0.3$.
The second statistic corresponds to the number of edges $|\ccalE|$.
We use $\bbbeta = \begin{bmatrix} 0.7 & -2\end{bmatrix}^\top$ as coefficients.
The graphs we generate have $n=50$ nodes, and we remove $|\ccalU| = 30$ values from $\bbA_0$.
We drop zeros and ones with equal probability so that we can use accuracy instead of F1 score for this experiment.
We generated $|\ccalA| = 1000$ graphs from this distribution to train the EDP-GNN and, to evaluate the impact the dataset size $|\ccalA|$ has on the estimation performance, we compare the edge prediction accuracy when fewer training samples are used.
The reported accuracies correspond to the average over $25$ runs, following the procedure described in Section~\ref{sec:hyperparams}.

\begin{figure}[ht]
     \centering
     \includegraphics[width=0.6\textwidth]{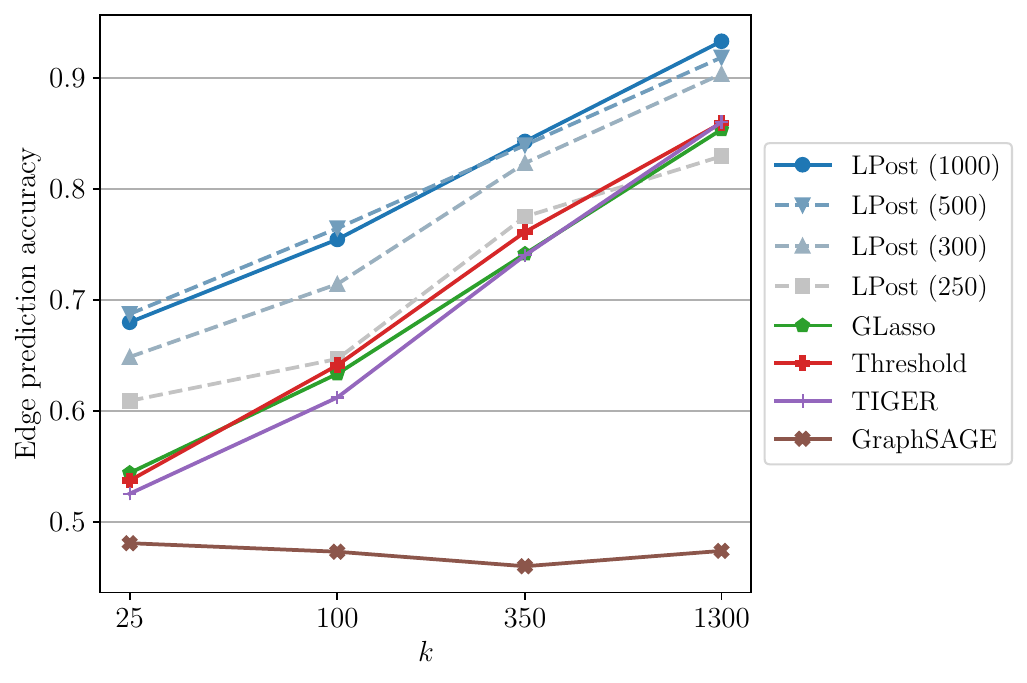}
     \caption{Prediction accuracy of several methods using ERGM graphs with $|\ccalU| = 30$. 
     The comparison includes five versions of Algorithm~\ref{alg:annealed_langevin}, each considering a prior that was learned with datasets $\ccalA$ of different sizes (the value of $|\ccalA|$ is indicated in the legend), as well as three benchmarks available in the literature.}
     \label{fig:ergm}
\end{figure}

Figure~\ref{fig:ergm} shows that the accuracy improves drastically as $|\ccalA|$ increases, as expected.
However, this enhancement seems to saturate as the dataset size approaches $|\ccalA| = 1000$.
All the other GGM estimation methods perform similarly, logically presenting a higher prediction accuracy as $k$ increases.
Notably, whenever $|\ccalA| \geq 300$, all other methods underperform our algorithm.

GraphSAGE performs poorly -- it predicts $0$ or $1$ uniformly at random.
Additionally, the incorporation of additional observations does not increase its accuracy.
As discussed in detail in the main paper, this behavior is expected and associated with GraphSAGE aggregating data using local neighborhoods, which is not a good fit for the setup at hand. 

\subsection{Performance dependence on \texorpdfstring{$|\ccalU|$}{|U|}} \label{sec:performance_u}

In order to understand how $|\ccalU|$ impacts the estimation performance, we run an additional experiment that complements the results presented in Section~\ref{sec:num_results} when using grid graphs with different $|\ccalU|$.
We use the test set of the ego-nets and fix $k=100$, varying the percentage of unknown values in $\bba$.
We run the simulations using three versions of Algorithm~\ref{alg:annealed_langevin}:
\begin{itemize}
    \item \textbf{Posterior (LPost)}. Our original approach, implementing all the steps in Algorithm~\ref{alg:annealed_langevin}.
    \item \textbf{Prior (LPr)}. This simplified version of Algorithm~\ref{alg:annealed_langevin} only exploits prior information. 
    This corresponds to running Algorithm~\ref{alg:annealed_langevin} where in line~\ref{lst:line:score_likelihood} we set $\nabla \log \ccalL_\bbX (\tbTheta_t) = 0$.
    \item \textbf{Likelihood (LL)}. This simplified version of Algorithm~\ref{alg:annealed_langevin} omits the learned prior and uses only the observations $\bbX$. 
    More specifically, this entails to run a version of Algorithm~\ref{alg:annealed_langevin} where in line~\ref{lst:line:score} we set $\bbg(\tbA_{t-1}, \sigma_l) = 0$.
\end{itemize}

We use the AUC score instead of F1 in this experiment to assess the prediction performance.
This allows us to avoid the tuning of $\tau_k$, which is unnecessary for this analysis since we are comparing three versions of Algorithm~\ref{alg:annealed_langevin} that return continuous predictions if no thresholding is applied.
Results are shown in Figure~\ref{fig:deezer_diff_missing}.

The main observations are: i) the information provided by $\ccalA$ decreases as $|\ccalU|$ increases, and ii) LPost consistently outperforms LPr and LL. 
Both are expected results that were observed in previous experiments. 
Analyzing more specific details, we note that the performance of LPr drops rapidly as $|\ccalU|$ increases, while that of LPost decays more slowly. 
The same is true for LL, whose AUC levels are stable and decay slowly. 
Finally, it is worth noticing that when a large portion of the graph is unknown, LPost approaches LL since the information provided by $|\ccalA|$ becomes less valuable.

\begin{figure}[ht]
     \centering
     \includegraphics[width=0.5\textwidth]{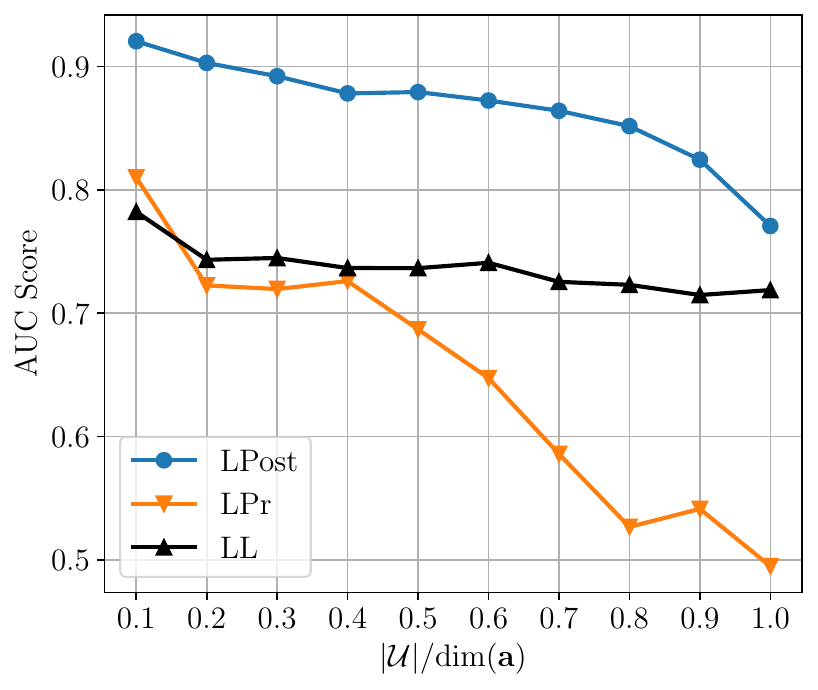}
     \caption{AUC score achieved by different versions Algorithm 1 when using an ego-net prior and $k=100$ observations. The horizontal axis represents different values of $|\ccalU|$. 
     For this experiment we skip the thresholding implemented in line~\ref{lst:line:threshold} of Algorithm~\ref{alg:annealed_langevin}.}
     \label{fig:deezer_diff_missing}
\end{figure}

\end{document}